\documentclass{article}

\usepackage{graphicx} 
\usepackage{subfigure}

\usepackage{algorithm}
\usepackage{algorithmic}

\usepackage{hyperref}

\usepackage{lmodern}
\usepackage[T1]{fontenc}

\DeclareFontFamily{T1}{lmsr}{}
\DeclareFontShape{T1}{lmsr}{m}{n}{<->ec-lmr5}{}
\newcommand*{\srmfamily}{\fontfamily{lmsr}\selectfont}
\DeclareTextFontCommand{\textsrm}{\srmfamily}

\usepackage{amssymb, multirow, paralist}
\usepackage{amsmath}
\usepackage{url}
\usepackage{amssymb}
\usepackage{amsthm}

\usepackage{algorithm,algorithmic}

\usepackage{verbatim}
\newtheorem{theorem}{Theorem}
\newtheorem{lemma}{Lemma}
\newtheorem{definition}[theorem]{Definition}
\newtheorem{remark}[theorem]{Remark}
\newtheorem{corollary}[theorem]{Corollary}

\def \R {\mathbb{R}}
\def \D {\mathcal{D}}

\def \E {\mathrm{E}}
\def \x {\mathbf{x}}

\def \L {\mathcal{L}}
\def \H {\mathcal{H}}

\def \v {\mathbf{v}}

\def \S {\mathcal{S}}

\def \z {\mathbf{z}}

\def \v {\mathbf{v}}
\def \w {\mathbf{w}}

\def \R {\mathbb{R}}

\def \A {\mathcal{A}}

\def \F {\mathcal{F}}

\def \wh {\widehat{\w}}

\def \lb {\L}
\def \Lh {\widehat{\L}}

\def \clip {\mbox{clip}}

\def \Xt {\widetilde{X}}

\def \A {\mathcal{A}}
\newcommand{\dd}[2] { \langle {#1}, {#2} \rangle}

\newcommand{\mc}[1] {\mathcal{#1}}

\def \eo {\epsilon_{\text{opt}}}
\def \ep {\epsilon_{\text{prior}}}

\def \w {\mathbf{w}}
\def \x {\mathbf{x}}

\def \R {\mathbb{R}}
\def \E {\mathbb{E}}

\def \z {\mathbf{z}}

\def \g {\hat{\mathbf{g}}}

\def \O {\mathcal{O}}

\usepackage{color,hyperref}
    \catcode`\_=11\relax
    
    \def\_email#1@#2\q_nil{%
      \href{mailto:#1@#2}{{\emailfont #1\emailampersat #2}}
    }
    \newcommand\emailfont{\sffamily}
    \newcommand\emailampersat{{\color{red}\small@}}
    \catcode`\_=8\relax

\begin{document}

\title{Passive Learning with Target Risk}

\date{}

\author{
    Mehrdad Mahdavi\\
    \small{Department of Computer Science}\\
    \small{Michigan State University}\\
    \small{\texttt{mahdavim@cse.msu.edu}}
  \and
    Rong Jin\\
    \small{Department of Computer Science}\\
    \small{Michigan State University}\\
    \small{\texttt{rongjin@cse.msu.edu}}
}

\maketitle
\begin{abstract}
In this paper we consider learning in passive setting but with a slight modification. We assume that the target expected loss, also referred to as target risk, is provided in advance for  learner as  prior knowledge. Unlike most studies in the learning theory  that only incorporate the prior knowledge into the generalization bounds, we are able to explicitly utilize the target risk in the learning process. Our analysis reveals a surprising result on the sample complexity of learning: by exploiting the target risk in the learning algorithm,  we show that when the loss function is both strongly convex and smooth, the sample complexity reduces to $\O(\log \left(\frac{1}{\epsilon}\right))$, an exponential improvement compared to the sample complexity $\O(\frac{1}{\epsilon})$ for learning with strongly convex loss functions.  Furthermore,  our proof is constructive and is based on a computationally efficient stochastic optimization algorithm  for such settings which demonstrate that the  proposed algorithm is practically useful.
\end{abstract}

\section{Introduction}
In the standard passive supervised  learning setting, the learning algorithm is given a set of labeled examples 
$\S = \left( (\x_1, y_1), \cdots, (\x_n, y_n) \right)$ drawn i.i.d. from a fixed but unknown  distribution $\D$. 
The goal, with the help of labeled examples, is to output a classifier $h$ from a predefined hypothesis class $\H$ 
that does well on unseen examples coming from the same distribution. The sample complexity of an algorithm is 
the number of examples which is sufficient to ensure that, with probability at least $1-\delta$ (w.r.t. the random choice of $\mc{S}$), the algorithm picks a hypothesis 
with an error that is at most $\epsilon$ from the optimal one. Sample complexity of passive learning is well established  and goes back to early works in the learning theory where the lower bounds  $\Omega\left(\frac{1}{\epsilon} ( \log \frac{1}{\epsilon}+ \log \frac{1}{\delta})\right)$ 
and $\Omega\left(\frac{1}{\epsilon^2} ( \log \frac{1}{\epsilon}+ \log \frac{1}{\delta})\right)$ were obtained  
in classic PAC and general agnostic PAC settings, respectively
~\cite{ehrenfeucht1989general,learnabilityvcdim89,anthony1999neural}.


In light of no free lunch theorem, learning is impossible unless we make assumptions regarding the nature of the problem  at hand. Therefore, when approaching a particular learning problem,  it is desirable to take into account some prior knowledge we might have about our problem and use a specialized algorithm that exploits this knowledge into a learning process or theoretical analysis. A key issue in this regard is the formalization of prior knowledge.  Such prior knowledge can be expressed by restricting our hypothesis class, making assumptions on the nature of unknown distribution $\D$ or formalization of the data space, analytical properties of the loss function being used to evaluate the performance, sparsity, and margin-- to name a few.

There has been an upsurge of interest  over the last decade in finding tight upper bounds on the sample complexity  by utilizing prior knowledge on the  analytical properties of the loss function,   that led to stronger generalization bounds  in agnostic PAC setting. In~\cite{DBLP:journals/tit/LeeBW98}  {\it fast} rates obtained for squared loss, exploiting the strong convexity of this loss function, which only holds under pseudo-dimensionality assumption.  With the recent development in online strongly convex optimization~\cite{hazanlog2006}, fast rates approaching $\O(\frac{1}{\epsilon} \log \frac{1}{\delta}) $ for convex Lipschitz strongly convex loss functions has been obtained in~\cite{fastrates2008,compl-linear-nips-2008}. For smooth non-negative loss functions,~\cite{srebro-2010-smoothness} improved the sample complexity  to {\it optimistic} rates
\[\O\left (\frac{1}{\epsilon}\left(\frac{\eo+\epsilon}{\epsilon} \right) \left( \log^3 \frac{1}{\epsilon}+ \log \frac{1}{\delta} \right)\right)\] 
for non-parametric learning using the notion of local Rademacher complexity~\cite{bartlett2005local}, where $\eo$ is the optimal risk.

In this work, we consider a slightly different setup for passive learning. We assume that before the start of the learning process, the learner has in mind a {\it target expected loss}, also referred to as {\it target risk},  denoted by $\ep$\footnote{We use $\ep$ instead of $\epsilon$ to emphasize the fact that this parameter is known to the learner in advance.}, and tries to learn a classifier with the expected risk of $O(\ep)$ by labeling a small number of training examples.  We further assume the target risk $\ep$ is feasible, i.e.,  $\ep \geq \eo$. To address this problem, we develop an efficient algorithm, based on stochastic optimization, for passive learning with target risk. The most surprising property of the proposed algorithm is that when the loss function is both smooth and strongly convex, it only needs $\O(d\log ({1}/{\ep}))$ labeled examples to find a classifier with the expected risk of $O(\ep)$, where $d$ is the dimension of data. This is a significant improvement compared to the sample complexity for empirical risk minimization. 


The key intuition behind our algorithm is that by knowing target risk as  prior knowledge, the learner has  better control over the variance in stochastic gradients, which contributes mostly to the slow convergence in stochastic optimization and consequentially large sample complexity in passive learning. The trick is to run  the stochastic optimization in multistages with a {\it fixed} size  and decrease the variance of stochastically perturbed gradients at each iteration by a properly designed mechanism.  Another crucial feature of the proposed algorithm is  to utilize the target risk $\ep$ to gradually refine the hypothesis space as the algorithm proceeds. Our algorithm differs significantly  from standard stochastic optimization algorithms and is able to achieve a geometric convergence rate with the knowledge of target risk $\ep$.

We note that our work does not contradict the lower bound in~\cite{srebro-2010-smoothness} because a {\it feasible} target risk $\ep$ is given in our learning setup and is fully exploited by the proposed algorithm. Knowing that the target risk $\ep$ is feasible makes it possible to improve the sample complexity from $\O({1}/{\ep})$ to $\O(\log({1}/{\ep}))$. We also note that although the logarithmic sample complexity is known for active learning~\cite{hanneke-thesis,mariatruesample2010}, we are unaware of any existing passive learning algorithm that is able to achieve a logarithmic sample complexity by incorporating any kind of prior knowledge.
\subsection{More Related Work}
\paragraph{Stochastic Optimization and Learnability}{Our work is related to the recent studies that examined the learnability from the viewpoint of stochastic convex optimization. In~\cite{sridharan-2012-learning,shalev-shwartz:2010:learnability}, the authors presented learning problems that are learnable by stochastic convex optimization but not by empirical risk minimization (ERM). Our work follows this line of research. The proposed algorithm achieves the sample complexity of $O(d\log(1/\ep))$ by explicitly incorporating the target expected risk $\ep$ into the stochastic convex optimization algorithm. It is however difficult to incorporate such knowledge into the framework of ERM. Furthermore, it is worth noting that in~\cite{ramdas-2013-optimal,sridharan-2012-learning,rakhlin-2010-online,Ben-DavidPS09}, the authors explored the connection between online optimization and statistical learning in the opposite direction. This was done by exploring the complexity measures developed in statistical learning for the learnability of online learning.}

\paragraph{Online and Stochastic Optimization}{The proposed algorithm is closely related to the recent works that stated  $O(1/n)$ is the optimal convergence rate for stochastic optimization when the objective function is strongly convex~\cite{primal-dual-nemirovsky,hazan-2011-beyond,sgd-suffic-icml2012}. In contrast, the proposed algorithm is able to achieve a geometric convergence rate for a target optimization error. Similar to the previous argument, our result does not contradict the lower bound given in~\cite{hazan-2011-beyond} because of the knowledge of a feasible optimization error. Moreover, in contrast to the multistage algorithm  in~\cite{hazan-2011-beyond} where the size of stages increases exponentially,  in our algorithm, the size of each stage  is fixed to be a constant.

}
\paragraph{Outline}{ The remainder of the paper is organized as follows: In Section \ref{sec:problem}, we set up notation, describe the setting, and discuss the assumptions on which  our algorithm relies. Section~\ref{sec:stochastic} motivates the problem and discusses the main intuition of our algorithm. The proposed algorithm and main result are discussed in Section \ref{sec:algorithm}. We prove the main result in Section \ref{sec:analysis}. Section \ref{sec:conclusion} concludes the paper  and the appendix contains the omitted proofs.}
\section{Preliminaries}
\label{sec:problem}
As usual in the framework of statistical learning theory, we consider a domain   $\mc{Z} := \mc{X} \times \mc{Y}$  where $\mc{X} \subseteq \R^d$ is the space for instances and $\mc{Y}$ is
the set of labels, and $\mc{H}$ is a hypothesis class.  We assume that the domain space $\mc{Z}$ is endowed with an unknown  Borel probability measure  $\mc{D}$.  We measure the performance of a specific hypothesis $h$ by defining a nonnegative loss function $\ell: \mc{H} \times \mc{Z} \rightarrow \mathbb{R}_{+}$.  We denote  the  risk of a hypothesis $h$ by  $\lb(h) = \mathbb{E}_{\z \sim \mc{D}}  [\ell(h, \z)]$. Given a sample $\mc{S} = (\z_1, \cdots, \z_n) = ( (\x_1, y_1), \cdots, (\x_n, y_n) ) \sim \mc{D}^n$, the goal of a learning algorithm is to pick a hypothesis $h: \mc{X} \rightarrow \mc{Y}$ from  $\mc{H}$ in such a  way that its risk $\lb(h)$ is close to the minimum possible risk of a hypothesis in $\mc{H}$.

Throughout this paper we pursue stochastic optimization viewpoint for risk minimization as detailed in Section \ref{sec:stochastic}. Precisely,  we focus on the  convex learning problems for which we assume that the hypothesis class $\H$ is a parametrized convex set $\H = \{h_{\w}: \x \mapsto \langle \w, \x \rangle: \w \in \R^d,   \|\w\| \leq R\}$ and for all $\z = (\x, y) \in \mc{Z}$, the loss function $\ell(\cdot, \z)$ is a non-negative convex function.  Thus, in the remainder we simply use vector $\w$ to represent $h_{\w}$, rather than working with  hypothesis $h_{\w}$. We will assume throughout that $\mc{X} \subseteq \R^d$ is the unit ball so that $\|\x\| \leq 1$. 
 Finally, the conditions under which we can get the desired result on sample complexity depend on analytic properties of the loss function. In particular, we assume that the loss function is strongly convex and smooth~\cite{nesterov-book}.
\begin{definition} [Strong convexity]A loss function $\ell(\w)$ is said to be $\alpha$-strongly convex w.r.t a norm $\|\cdot\|$\footnote{Throughout this paper, we only consider the $\ell_2$-norm.}, if  there exists a constant $\alpha > 0$ (often called the modulus of strong convexity) such that, for any $\lambda\in[0, 1]$ and for all $\w_1,\w_2\in \H$, it holds  that
\begin{equation*}
\ell(\lambda\w_1+ (1-\lambda)\w_2)\leq \alpha \ell(\w_1) + (1-\lambda) \ell(\w_2) - \frac{1}{2}\lambda(1-\lambda)\alpha\|\w_1-\w_2\|^2.
\end{equation*}
\end{definition}
When $\ell(\w)$ is differentiable, the strong convexity is equivalent to
\[\ell(\w_1) \geq \ell(\w_2) + \langle \nabla \ell(\w_2), \w_1-\w_2\rangle + \frac{\alpha}{2}\|\w_1-\w_2\|^2,\; \forall \;\w_1,\w_2\in\H.\] We would like to emphasize that  in our setting, we only need  that the expected  loss function $\L(\w)$ be strongly convex, without having to assume strong convexity for individual loss functions.  \\
Another property  of loss function that underline our analysis is its smoothness. Smooth functions arise, for instance, in logistic and least-squares regression,  and in general for learning linear predictors where the loss function has a Lipschitz-continuous gradient. 
\begin{definition} [Smoothness]A differentiable  loss function $\ell(\w)$ is said to be  $\beta$-smooth  with respect to a norm $\|\cdot\|$, if it holds that
\begin{equation}
\label{eqn:smoth}
\ell(\w_1) \leq \ell(\w_2) +  \dd{\nabla\ell(\w_2)}{\w_1-\w_2} + \frac{\beta}{2}\|\w_1-\w_2\|^2, \;\forall \; \w_1,\w_2\in\H.
\end{equation}
\end{definition}

\section{The Curse of Stochastic Oracle}
\label{sec:stochastic}

We begin by discussing stochastic optimization  for risk minimization, convex learnability,  and  then the main intuition that motivates this work. 

Most existing learning algorithms follow the framework of empirical risk minimizer (ERM) or regularized ERM,  which was developed to great extent by Vapnik and Chervonenkis~\cite{vapnik1971uniform}.  Essentially, ERM methods use the empirical loss  over $\mc{S}$, i.e., $\Lh(\w) = \frac{1}{n} \sum_{i = 1}^{n}{\ell(\w, \z_i)}$, as a criterion to pick a hypothesis. In regularized ERM methods, the learner  picks a hypothesis that jointly minimizes $\Lh(\w) $ and a regularization function over $\w$.  We note that ERM resembles the widely used Sample Average Approximation (SAA) method in the optimization community  when the hypothesis space and the loss function are convex. If uniform convergence holds, then the empirical risk minimizer is consistent, i.e., the population risk of the ERM converges to the optimal population risk, and the problem is learnable using  ERM.

A rather different paradigm for risk minimization is stochastic optimization. Recall that the goal of learning is to  approximately minimize the risk  $\lb(\w) = \mathbb{E}_{\z \sim \D}[\ell(\w, \z)]$.  However,  since  the distribution $\D$ is unknown to the learner, we can not utilize standard gradient methods to minimize the expected loss.  Stochastic optimization methods  circumvent this problem by allowing the optimization method to take a step which is only in expectation along the negative of the gradient.  To motivate stochastic optimization as an alternative to the ERM method,~\cite{shalev2009stochastic,shalev2009learnability} challenged the ERM method and showed that there  is a real gap between learnability and uniform convergence by investigating non-trivial problems where no uniform convergence holds, but they are still learnable using Stochastic Gradient Descent (SGD) algorithm~\cite{nemirovski2009robust}. These results  uncovered an important relationship between learnability and stability, and showed that stability  together with approximate empirical risk minimization, assures learnability~\cite{shalev-shwartz:2010:learnability}. We note that Lipschitzness  or smoothness of loss function is necessary for an algorithm to be stable, and boundedness and convexity  alone are not sufficient for ensuring that the convex learning problem is learnable.

To directly solve $\min_{\w \in \mc{H}} \lb({\w}) = \E_{\z \sim \D}[\ell(\w, \z)]$, a typical stochastic  optimization algorithm initially picks some point in the feasible set $\mc{H}$   and  iteratively updates these points based on first order perturbed gradient information about the function at those points.  For instance, the widely used SGD algorithm  starts with $\w_0 = \mathbf{0}$; at each iteration $t$, it queries the stochastic oracle ($\mathcal{SO}$) at $\w_t$ to obtain a perturbed but unbiased gradient $\g_t$ and updates the current solution by
\[ \w_{t+1} = {\Pi}_{\mc{H}} \left(\w_t - \eta_t \g_t\right),\]
where $\Pi_{\H}(\w)$ projects the solution $\w$ into the domain $\H$. To capture the efficiency of optimization procedures in a general sense, one can use  oracle complexity of the algorithm which, roughly speaking,  is the minimum number of calls to any  oracle needed by any method to achieve desired accuracy~\cite{nesterov-book}. We note that the oracle complexity corresponds to the sample complexity of learning from the stochastic optimization viewpoint  previously discussed. The following theorem states a lower bound on the sample complexity of stochastic optimization algorithms~\cite{nemirovsky1983problem}.
\begin{theorem}[Lower Bound on Oracle Complexity] Suppose $\lb({\w}) = \E_{\z \sim \D}[\ell(\w, \z)]$ is  $\alpha$-strongly and $\beta$-smooth convex  function defined over convex domain $\mc{H}$. Let $\mathcal{SO}$ be a stochastic oracle that for any point $\w \in \mc{H}$ returns an unbiased estimate $\g$, i.e.,  $\E[\g] = \nabla \lb(\w)$, such that $\E\left[\|\g-\nabla \lb(\w)\|^2\right] \leq \sigma^2$ holds. Then for any stochastic optimization algorithm $\A$ to find a solution $\wh$ with $\epsilon$ accuracy respect to the optimal solution $\w_*$, i.e.,  $\E \left[ \lb(\wh) - \lb(\w_*) \right] \leq \epsilon$,  the number of calls  to $\mc{SO}$ is lower bounded by

\begin{eqnarray}
\mathcal{O}(1) \left( \sqrt{\frac{\beta}{\alpha}} \log \left( \frac{\beta \| \w_0 - \w_*\|^2}{\epsilon}\right) + \frac{\sigma^2}{\alpha \epsilon}\right).
\label{eqn:lower}
\end{eqnarray}
\label{thm:lower}
\end{theorem}
The first term in~(\ref{eqn:lower}) comes from deterministic oracle complexity and the second term is due to noisy gradient information provided by $\mc{SO}$. As indicated in~(\ref{eqn:lower}), the slow convergence rate for stochastic optimization is due to the variance in stochastic gradients, leading to at least $\mathcal{O}\left({\sigma^2}/{\epsilon}\right)$ queries to be issued. We note that the idea of mini-batch~\cite{ohadminibatch,duchirandomizedsmooth}, although it reduces the variance in stochastic gradients, does not reduce the oracle complexity.

We close this section by  informally presenting why logarithmic sample complexity is, in principle, possible,  
under the assumption that target risk is known to the learner $\A$. To this end, consider the setting of 
Theorem~\ref{thm:lower} and assume that the learner $\A$ is given the prior accuracy $\ep$ and is asked to find an  $\ep$-accurate solution. If it happens that the variance of $\mc{SO}$ has the same magnitude as $\ep$, i.e., 
$\E\left[\|\g-\nabla \lb(\w)\|^2\right] \leq \ep$,  then from~(\ref{eqn:lower}) it follows that  the second term 
vanishes and the learner $\A$  needs to issue only $\mathcal{O} \left(\log {1}/{\ep}\right)$ queries to find the solution.
But, since there is no control on  $\mc{SO}$, except that  the variance of stochastic gradients are bounded,  $\A$ needs a mechanism to manage the variance of perturbed gradients 
at each iteration in order to alleviate the influence of noisy gradients. One strategy is to replace the unbiased estimate of gradient with a biased one, 
which unfortunately may yield loose bounds.  To overcome this problem, we introduce a strategy that shrinks the solution space with respect to the target risk $\ep$ to
control the damage caused by biased estimates.
\section{Algorithm and Main Result}
\label{sec:algorithm}
 In this section we proceed to describe the proposed algorithm and state the main result on its sample complexity.

\subsection{Description of Algorithm}
We now turn to describing our algorithm.  Interestingly, our algorithm is quite dissimilar to the classic stochastic optimization methods. It proceeds by running the algorithm
online on fixed chunks of examples, and using the intermediate hypotheses and target risk $\ep$ to gradually refine the hypothesis  space. As  mentioned above, we assume in our setting that the target expected risk $\ep$ is  provided to the learner a priori. We further assume the target risk $\ep$ is feasible for the solution within the domain $\H$, i.e., $\ep \geq \eo$.  The proposed algorithm explicitly takes advantage of the  knowledge of expected risk $\ep$ to attain an $O\left(\log(1/\ep)\right)$ sample complexity.

Throughout we shall consider linear predictors of form $\dd{\w}{\x}$ and assume that the loss function of interest   $\ell (\dd{\w}{\x}, y)$  is $\beta$-smooth.  It is straightforward to see that $\lb(\w) = \E_{(\x, y)\sim\D}\left[ \ell(\dd{ \w}{\x},y) \right]$ is also $\beta$-smooth. In addition to the smoothness of the loss function, we also assume that $\lb(\w)$ to be $\alpha$-strongly convex. We denote by $\w_*$ the optimal solution that minimizes $\lb(\w)$, i.e., $\w_* = \mathop{\arg\min}_{\w \in \H} \lb(\w)$, and denote its optimal value by $\eo$.

Let $(\x_t, y_t), t = 1, \ldots, T$ be a sequence of i.i.d. training examples.  The proposed algorithm  divides the $T$ iterations into the $m$ stages, where each stage consists of $T_1$ training examples, i.e., $T = m T_1$. Let $(\x_k^t, y_k^t)$ be the $t$-th training example received at stage $k$, and let $\eta$ be the step size used by all the stages. At the beginning of each stage $k$, we initialize the solution $\w$ by the average solution $\wh_k$ obtained from the last stage, i.e.,
\begin{eqnarray}
    \wh_k = \frac{1}{T_1} \sum_{t=1}^{T_1} \wh_k^t,  \label{eqn:average}
\end{eqnarray}
where $\wh_k^t$ denotes the $t$th solution at stage $k$. Another feature of the proposed algorithm is a domain shrinking strategy that adjusts the domain as the algorithm proceeds using intermediate hypotheses and target risk.  We  define the domain $\H_k$ used at stage $k$ as
\begin{eqnarray}
    \H_k = \left\{\w \in \H: \|\w - \wh_k\| \leq \Delta_k \right\}, \label{eqn:omega}
\end{eqnarray}
where $\Delta_k$ is the domain size,  whose value will be discussed later. Similar to the SGD method, at each iteration of stage $k$, we receive a training example $(\x_k^t, y_k^t)$, and compute the gradient $\g_k^t = \ell'\left(\dd{\w_k^t}{\x_k^t}, y_t\right) \x_k^t$. Instead of using the gradient directly, following~\cite{DBLP:journals/corr/abs-1108-4559}, a clipped version of the gradient, denoted by $\v_k^t = \clip\left(\gamma_k, \g_k^t\right)$, will be used for updating the solution. More specifically, the clipped vector $\v_k^t \in \R^d$ is defined as
\begin{eqnarray}
    [\v_k^t]_i = \clip\left(\gamma_k, \left[  \g_k^t \right]_i\right) = \mbox{sign}\left(\left[ \g_k^t \right]_i\right)\min\left(\gamma_k, \left| \left[ \g_k^t \right]_i\right|\right), i=1, \ldots, d \label{eqn:clip}
\end{eqnarray}
where $\gamma_k = 2\xi\beta\Delta_k$ with $\xi \geq 1$. Given the clipped gradient $\v_k^t$, we follow the standard framework of stochastic gradient descent, and update the solution by
\begin{eqnarray}
    \w_k^{t+1} = \Pi_{\H_k}\left(\w_k^t - \eta \v_k^t \right). \label{eqn:update}
\end{eqnarray}
\begin{algorithm}[t]
\label{alg:1}
\caption{Convex Learning with Target  Risk}
\begin{algorithmic}[1]
\STATE {\bf Input:} step size $\eta$, stage size $T_1$, number of stages $m$, target expected risk $\ep$, parameters $\varepsilon \in (0, 1)$ and $\tau \in (0, 1)$ used for updating domain size $\Delta_k$, and parameter $\xi \geq 1$ used to clip  the gradients \\ 
\STATE {\bf Initialization:} $\wh_1 = 0$, $\Delta_1 = R$, and $\H_1 = \H$

\FOR{$k = 1, \ldots, m$}
    \STATE Set $\w_k^t = \wh_k$ and $\gamma_k = 2\xi\beta\Delta_k$
    \FOR{$t=1, \ldots, T_1$}
        \STATE Receive training example $(\x_t, y_t)$
        \STATE Compute the gradient $\g_k^t$ and the clipped version of the gradient $\v_k^t$ using Eq.~(\ref{eqn:clip})
        \STATE Update the solution $\w_k^t$ using Eq.~(\ref{eqn:update}).
    \ENDFOR
    \STATE Update $\Delta_k$ using Eq.~(\ref{eqn:delta}).
    \STATE Compute the average solution $\wh_{k+1}$ according to Eq.~(\ref{eqn:average}), and update the domain $\H_{k+1}$ using the expression in~(\ref{eqn:omega}).
\ENDFOR
\end{algorithmic} \label{alg:1}
\end{algorithm}

The purpose of introducing the clipped version of the gradient is to effectively control the variance in stochastic gradients, an important step toward achieving the geometric convergence rate. At the end of each stage, we will update the domain size by explicitly exploiting the target expected risk $\ep$ as
\begin{eqnarray}
\Delta_{k+1}= \sqrt{\varepsilon \Delta_k^2 + \tau \ep} \label{eqn:delta}\;,
\end{eqnarray}
where $\varepsilon \in (0, 1)$ and $\tau \in (0, 1)$ are two parameters, both of which  will be discussed later. 

Algorithm~\ref{alg:1} gives the detailed steps for the proposed method. The three  important aspects of Algorithm~\ref{alg:1}, all  crucial  to  achieve  a geometric convergence rate, are highlighted as follows:
\begin{itemize}
\item Each stage of the proposed algorithm is comprised of the same number of training examples. This is in contrast to the epoch gradient algorithm~\cite{hazan-2011-beyond} which divides $m$ iterations into exponentially increasing epochs, and runs SGD with averaging on each epoch.  Also, in our case the learning rate is fixed for all iterations.
\item The proposed algorithm uses a clipped gradient for updating the solution in order to better control the variance in stochastic gradients; this stands in contrast to  the SGD method, which uses original gradients  to update the solution.
\item The proposed algorithm takes into account the targeted expected risk and intermediate  hypotheses when updating the domain size at each stage. The purpose of domain shrinking  is to reduce the damage caused by biased gradients that resulted from clipping operation.
\end{itemize}
\subsection{Main Result on Sample Complexity}
The main theoretical result of Algorithm~\ref{alg:1} is given in the following theorem.
\begin{theorem}[Convergence Rate]
\label{thm:main}
Assume that the hypothesis space $\H$ is compact and the loss function $\ell$ is $\alpha$-strongly convex and $\beta$-smooth.  Let $T = m T_1$ be the size of  the sample and $\epsilon_{\rm{prior}}$ be the target expected loss given to the learner in advance such that  $\epsilon_{\rm{opt}} \leq \epsilon_{\rm{prior}}$ holds. Given $\varepsilon \in (0, 1)$ and $\tau \in (0, 1)$, set $\xi$, $\eta$, and $T_1$ as
\begin{eqnarray*}
\xi = \frac{4\beta}{\alpha \tau}, \; T_1 = 4\max\left\{\frac{\xi^3\beta d + 2\xi \beta\sqrt{d}}{\varepsilon \alpha}\ln\frac{ms}{\delta}, \frac{16\xi^2\beta^2}{\alpha^2\varepsilon^2} \right\}, \; \eta = \frac{1}{2\xi\beta\sqrt{T_1}}, \label{eqn:BT}
\end{eqnarray*}
where
\begin{eqnarray}
    s = \left\lceil \log_2 \frac{\xi \beta R^2}{\epsilon_{\rm{prior}}} \right\rceil. \label{eqn:s}
\end{eqnarray}
After running Algorithm~\ref{alg:1} over $m$ stages, we have, with a probability $1 - \delta$,
\[
    \L(\wh_{m+1}) \leq \frac{\beta R^2}{2}\varepsilon^m + \left(1 + \frac{\tau}{1 - \varepsilon}\right){\epsilon_{\rm{prior}}},
\]
implying that only $O(d\log [1/\epsilon_{\rm{prior}}])$ training examples are needed in order to achieve a risk of $O(\epsilon_{\rm{prior}})$.
\end{theorem}

We note that comparing to the bound in Theorem~\ref{thm:lower}, for  Algorithm~\ref{alg:1}  the level of error to which the linear convergence holds is not determined by the noise level in stochastic gradients, but by the target risk. In other words, the algorithm is able to tolerate the noise by knowing the target risk as prior knowledge and achieves a linear convergence to the level of the target risk even when the variance of stochastic gradients is much larger than the target risk.  In addition, although the result given in Theorem~\ref{thm:main} assumes a bounded domain with $\|\w\| \leq R$, however, this assumption can be lifted by effectively exploring the strong convexity of the loss function and further assuming that the loss function is Lipschitz continuous with constant $G$, i.e.,  $|\L(\w_1)- \L(\w_2)|\leq G\|\w_1-\w_2\|,\; \forall\; \w_1,\w_2\in\H$. More specifically, the fact that the $\L(\w)$ is $\alpha$-strongly convex with first order optimality condition, for the optimal solution $\w_* = \arg \min_{\w \in \mc{H}} \L(\w)$, we have
\[\L(\w) - \L(\w_*) \geq \frac{\alpha}{2}\|\w - \w_* \|^2, \;\; \forall \w \in \mc{H}.\]
This inequality combined with Lipschitz continuous assumption implies that for any $\w \in \H$ the inequality $\|\w-\w_*\| \leq R_* := 2G/\alpha$ holds, and therefore we can simply set $R = R_*$. We also note that  this dependency can  be resolved  with a weaker assumption than Lipschitz continuity,  which only depends on the gradient of loss function at origin.  To this end,  we define $|\ell'(0, y)| = G$. Using the fact that $\L(\w)$ is $\alpha$-strongly, it is easy to verify that $\frac{\alpha}{2} \|\w_*\|^2 - G \|\w_*\| \leq 0$, leading to $\|\w_*\| \leq R_* := \frac{2}{\alpha}G$ and, therefore, we can simply set $R = R_*$. 

We  now use our analysis of Algorithm~\ref{alg:1} to obtain a sample complexity analysis
for learning  smooth strongly convex problems with a bounded hypothesis class. To make it easier to parse, we only keep the dependency on the main parameters  $d$, $\alpha$, $\beta$, $T$, and $\ep$ and hide the dependency on other constants in $\O(\cdot)$ notation. Let $\wh$ denote the output of Algorithm~\ref{alg:1}. By setting $\varepsilon = 0.5$ and  letting $c = O(\tau)$ to be an arbitrary small number, Theorem~\ref{thm:main} yields the following:

\begin{corollary}[Sample Complexity]  Under the same conditions as Theorem~\ref{thm:main}, by running  Algorithm~\ref{alg:1} for minimizing $\L(\w)$ with a number of iterations (i.e., number of training examples)  $T$, if it holds that,

\[ T \geq \O\left(d \kappa^4 \left(\log \frac{1}{\epsilon_{\rm{prior}}}\log \log \frac{1}{\epsilon_{\rm{prior}}} + \log\frac{1}{\delta} \right) \right)\]  
where $\kappa = \beta/\alpha$ denotes the condition number of the loss function and $d$ is the dimension of data,  then with a probability $1 - \delta$,   $\wh$ attains a risk of $O(\epsilon_{\rm{prior}})$, i.e., $\L(\wh) \leq (1+c) \epsilon_{\rm{prior}}$.
\label{corollary:sample}
\end{corollary}

As an example of a concrete problem that may be put into the setting of the present work is the regression problem with squared loss. It is easy to show that average square loss function is Lipschitz continuous with a Lipschitz constant  $\beta = \lambda_{\max} (X^{\top}X)$ which denotes the largest eigenvalue of matrix $X^{\top}X$ where $X$ is the data matrix. The strong convexity is guaranteed as long as  the population data covariance matrix is not rank-deficient and its  minimum eigenvalue  is lower bounded by a constant $\alpha>0$. For this problem, the optimal minimax sample complexity is known to be $O(\frac{1}{\epsilon})$, but as it implies from Corollary~\ref{corollary:sample}, by the knowledge of target risk $\ep$, it is possible to reduce the sample complexity  to $O(\log (1/{\epsilon_{\rm{prior}}}))$.

\begin{remark} It is indeed remarkable that the sample complexity of Theorem~\ref{thm:main} has $\kappa^4 = \left(\beta/\alpha\right)^4$ dependency on the condition number of the loss function, which is worse than the $\sqrt{{\beta}/{\alpha}}$ dependency in the lower bound in (\ref{eqn:lower}). Also, the explicit dependency of sample complexity on  dimension $d$ makes the proposed algorithm inappropriate for non-parametric settings.
\end{remark}
\section{Analysis}
\label{sec:analysis}
Now we turn to proving  the main theorem. The proof will be given in a series of lemmas and theorems where the proof of few are given in the appendix. The proof makes use of the Bernstein inequality for martingales, idea of peeling process, self-bounding property of smooth loss functions, standard analysis of stochastic optimization, and novel ideas to derive the claimed sample complexity for the proposed algorithm.

The proof of  Theorem~\ref{thm:main} is by induction and we start with the key step  given in the following theorem.
\begin{theorem}
\label{thm:induction}
Assume $\epsilon_{\rm{prior}} \geq \epsilon_{\rm{opt}}$. For a fixed stage $k$, if $\|\wh_k - \w_*\| \leq \Delta_k$, then, with a probability $1 - \delta$, we have
\[
    \|\wh_{k+1} - \w_*\|^2 \leq a \Delta_k^2 + b \epsilon_{\rm{prior}}
\]
where
\begin{eqnarray}
    a = \frac{2}{\alpha T_1}\left(2\xi\beta\sqrt{T_1} + \left[\xi^3\beta d + 2\xi\beta\sqrt{d}\right]\ln\frac{s}{\delta} \right), \quad b = \frac{8}{\alpha \xi} \label{eqn:ab}
\end{eqnarray}
and $s$ is given in (\ref{eqn:s}), provided that $\xi \geq 16\beta/\alpha$ and $\eta = 1/(2\xi\beta\sqrt{T_1})$ hold.
\end{theorem}

Taking this statement as given for the moment, we proceed with the proof of Theorem~\ref{thm:main}, returning later to establish the claim stated in  Theorem~\ref{thm:induction}. \\
\begin{proof}[of Theorem~\ref{thm:main}]
By setting $a$ and $b$ in (\ref{eqn:ab}) in Theorem~\ref{thm:induction} as $a \leq \varepsilon$ and $ b \leq {2\tau}/{\beta}$, we have $\xi \geq 4\beta/(\alpha\tau)$ and
\[
T_1 \leq \frac{2}{\alpha\varepsilon}\left(2\xi\beta\sqrt{T_1} + \left[\xi^3\beta d + 2\xi\beta\sqrt{d}\right]\ln\frac{s}{\delta} \right)
\]
implying that
\[
T_1 \geq 4\max\left\{\frac{\xi^3\beta d + 2\xi\beta\sqrt{d}}{\varepsilon \alpha}\ln\frac{s}{\delta}, \frac{16\xi^2\beta^2}{\alpha^2\varepsilon^2} \right\}.
\]
Thus, using Theorem~\ref{thm:induction} and the definition of $\xi$ and $T_1$, we have, with a probability $1 - \delta$,
\[
    \Delta^2_{k+1}  \leq \varepsilon \Delta_k^2 + \frac{2 \tau}{\beta} \ep.
\]
After $m$ stages, with a probability $1 - m\delta$, we have
\[
    \Delta^2_{m+1} \leq \varepsilon^{m}\Delta^2_1 + \frac{2\tau}{\beta}\ep\sum_{i=0}^{m-1} \varepsilon^i \leq \varepsilon^m\Delta^2_1 + \frac{2\tau}{\beta(1 - \varepsilon)}\ep.
\]
By the $\beta$-smoothness of $\lb(\w)$,  it implies that
\begin{eqnarray*}
\lb (\wh_{m+1}) - \lb(\w_*) \leq \frac{\beta}{2}\| \wh_{m+1} - \w_*\|^2 
 &\leq&  \frac{\beta}{2}\varepsilon^m \Delta^2_1 + \frac{\tau}{1 - \varepsilon} \ep,\\
 &\leq& \frac{\beta R^2}{2}\varepsilon^m +  \frac{\tau}{1 - \varepsilon} \ep,
\end{eqnarray*}
where  the last inequality follows from  $\Delta_1 \leq  {R}$.  The bound stated in the theorem follows the assumption that $\lb(\w_*) = \eo \leq \ep$.
\end{proof}

\subsection{Proof of Theorem~\ref{thm:induction}}
To bound $\|\wh_{k+1} - \w_*\|$ in terms of $\Delta_k$, we start with the standard analysis of online learning. In particular, from the strong convexity assumption of $\lb (\w)$ and updating rule (\ref{eqn:update}) we have,
\begin{eqnarray}
 \label{eqn:1}
  \lb(\w_k^t) - \lb(\w_*) &\leq& \langle \nabla \lb(\w_k^t), \w_k^t - \w_* \rangle - \frac{\alpha}{2}\|\w_k^t - \w_*\|^2 \nonumber \\
&=&  \langle \v_k^t, \w_k^t - \w_* \rangle + \langle \nabla \lb(\w_k^t) - \v_k^t , \w_k^t - \w_* \rangle - \frac{\alpha}{2}\|\w_t - \w_*\|^2 \nonumber \\
& \leq&  \frac{\|\w_k^{t+1} - \w_*\|^2 - \|\w_k^{t+1} - \w_*\|^2}{2\eta} + \frac{\eta d}{2}\gamma_k^2  \nonumber \\  &  & + \underbrace{\langle \nabla \lb(\w_k^t) - \v_k^t , \w_k^t - \w_* \rangle}\limits_{\triangleq v_k^t} - \frac{\alpha}{2}\|\w_t - \w_*\|^2,
\end{eqnarray}
where the last step follows from $\|\v_k^t\| \leq \gamma_k \sqrt{d}$.
By adding all the inequalities of (\ref{eqn:1}) at stage $k$, we have
\begin{eqnarray}
\sum_{t=1}^{T_1} \lb(\w_k^t) - \lb(\w_*) & \leq & \frac{\|\wh_k - \w_*\|^2}{2\eta} + \frac{d\eta}{2}\gamma_k^2 T_1 + \sum_{t=1}^{T_1} v_k^t - \frac{\alpha}{2}\sum_{t=1}^{T_1} \|\w_t - \w_*\|^2 \nonumber \\
& \leq & \frac{\Delta_k^2}{2\eta} + \frac{d\eta}{2}\gamma_k^2 T_1 + V_k - \frac{\alpha}{2}W_k, \label{eqn:2}
\end{eqnarray}
where $V_k$ and $W_k$ are defined as $V_k = \sum_{t=1}^{T_1} v_k^t$ and $W_k = \sum_{t=1}^{T_1} \|\w_k^t - \w_*\|^2$, respectively.
In order to bound $V_k$, using the fact that $\nabla \lb(\w_k^t) = \E_t[\g_k^t]$, we rewrite $V_k$ as
\begin{eqnarray*}
V_k & = & \sum_{t=1}^{T_1} \underbrace{\langle - \v_k^t + \E_t[\v_k^t], \w_k^t - \w_* \rangle}_{\triangleq d_k^t} + \sum_{t=1}^{T_1} \underbrace{\langle \E_t\left[\g_k^t\right] - \E_t[\v_k^t], \w_k^t - \w_* \rangle}_{\triangleq e_k^t}   \\
& = & D_k + E_k,
\end{eqnarray*}
where $D_k = \sum_{t=1}^{T_1} d_k^t$ and $E_k = \sum_{t=1}^{T_1} e_k^t$ which represent the variance and bias of the clipped gradient $\v_k^t$, respectively. We now turn to separately upper bound each term.

The following lemma bounds the variance term $D_k$ using the Bernstein inequality for martingale. Its proof can be found in Appendix \ref{app:lemma1}.
\begin{lemma} \label{lemma:d}
For any $L > 0$ and $\mu > 0$, we have
\begin{eqnarray*}
    \Pr\left(W_k \leq \frac{\epsilon_{\rm{prior}} T_1}{2\mu\beta}\right) + \Pr\left(D_k \leq \frac{1}{L}W_k + \left(L\gamma_k^2 d + \gamma_k\Delta_k\sqrt{d}\right)\ln\frac{s}{\delta} \right) \geq 1 - \delta
\end{eqnarray*}
where $s$ is given by
\[
    s = \left\lceil \log_2 \frac{8\beta\mu R^2}{\epsilon_{\rm{prior}}} \right\rceil.
\]
\end{lemma}
The following lemma bounds $E_k$ using the self-bounding property of smooth functions and the proof is deferred to  Appendix \ref{app:lemma2}.
\begin{lemma} \label{lemma:e}
\[
E_k \leq \frac{4T_1}{\xi}\epsilon_{\rm{opt}} + \frac{4\beta}{\xi} W_k \leq \frac{4T_1}{\xi}\epsilon_{\rm{prior}} + \frac{4\beta}{\xi} W_k.
\]
\end{lemma}
Note that  without the knowledge of $\ep$, we have to bound $\eo$ by $\Omega(1)$, resulting in a very loose bound for the bias term $E_k$. It is knowledge of the target expected risk $\ep$ that allows us to come up with a significantly more accurate bound for the bias term $E_k$, which consequentially leads to a geometric convergence rate.

We now proceed to bound $\sum_{t=1}^{T_1} \lb(\w_k^t) - \lb(\w_*)$ using the two bounds in Lemma~\ref{lemma:d} and \ref{lemma:e}. To this end, based on the result obtained in Lemma~\ref{lemma:d}, we consider two scenarios. In the first scenario, we assume
\begin{eqnarray}
    W_k \leq \frac{\ep T_1}{2\mu\beta} \label{eqn:condition-1}
\end{eqnarray}
In this case, we have
\begin{eqnarray}
    \sum_{t=1}^{T_1} \lb(\w_k^t) - \lb(\w_*) \leq \frac{\beta}{2}W_k \leq \frac{\ep}{2\mu} T_1. \label{eqn:bound-1}
\end{eqnarray}
In the second scenario, we assume
\begin{eqnarray}
    D_k \leq \frac{1}{L}W_T + \left(L\gamma_k^2 d + \gamma_k\Delta_k\sqrt{d}\right)\ln\frac{s}{\delta}. \label{eqn:condition-2}
\end{eqnarray}
In this case, by combining the bounds for $D_k$ and $E_k$ and setting $L = \frac{\xi}{4\beta}$, we have
\begin{eqnarray*}
    V_k & \leq & \frac{8\beta}{\xi} W_k + \left(\frac{\xi d}{4\beta}\gamma_k^2 + \gamma_k\Delta_k\sqrt{d}\right)\ln\frac{s}{\delta} + \frac{4T_1}{\xi}\ep \\
    & = & \frac{8\beta}{\xi} W_k + \left(\xi^3\beta d + 2\xi \beta\sqrt{d}\right)\Delta_k^2\ln\frac{s}{\delta} + \frac{4T_1}{\xi}\ep,
\end{eqnarray*}
where the last equality follows from the fact $\gamma_k = 2\xi \beta \Delta_k$. If we choose $\xi$ such that $\frac{8\beta}{\xi} \leq \frac{\alpha}{2}$ or $\xi \geq\frac{16\beta}{\alpha} > 1$ holds, we get
\begin{eqnarray*}
    V_k \leq \frac{\alpha}{2}W_k + \left(\xi^3\beta d+ 2\xi\beta\sqrt{d}\right)\Delta_k^2\ln\frac{s}{\delta} + \frac{4T_1}{\xi}\ep
\end{eqnarray*}
Substituting the above bound for $V_k$ into the inequality of (\ref{eqn:2}), we have
\[
\sum_{t=1}^{T_1} \lb(\w_k^t) - \lb(\w_*) \leq \frac{\Delta_k^2}{2\eta} + \frac{\eta}{2}\gamma_k^2 T_1 + \left(\xi^3\beta d+ 2\xi \beta\sqrt{d}\right)\Delta_k^2\ln\frac{s}{\delta} + \frac{4T_1}{\xi}\ep
\]
By choosing $\eta$ as $\eta = \frac{\Delta_k}{\gamma_k\sqrt{T_1}} = \frac{1}{2\xi\beta\sqrt{T_1}}$, we have
\begin{eqnarray}
\lb(\wh_{k+1}) - \lb(\w_*) \leq \frac{1}{T_1}\left(2\xi\beta\sqrt{T_1}+\left[\xi^3\beta d+ 2\xi\beta\sqrt{d}\right]\ln\frac{s}{\delta}\right)\Delta^2_k + \frac{4}{\xi}\ep. \label{eqn:bound-2}
\end{eqnarray}
By combining the bounds in (\ref{eqn:bound-1}) and (\ref{eqn:bound-2}), under the assumption that at least one of the two conditions in (\ref{eqn:condition-1}) and (\ref{eqn:condition-2}) is true, by setting $\mu = B/8$, we have
\begin{eqnarray*}
\lb(\wh_{k+1}) - \lb(\w_*) \leq \frac{1}{T_1}\left(2\xi\beta\sqrt{T_1}+\left[\xi^3\beta d+ 2\xi\beta\sqrt{d}\right]\ln\frac{s}{\delta}\right)\Delta^2_k + \frac{4}{\xi}\ep,
\end{eqnarray*}
implying
\begin{eqnarray*}
 \|\wh_{k+1} - \w_*\| \leq \frac{2}{\alpha T_1}\left(2\xi\beta\sqrt{T_1}+\left[\xi^3\beta d+ 2\xi\beta\sqrt{d}\right]\ln\frac{s}{\delta}\right)\Delta^2_k + \frac{8}{\alpha \xi}\ep.
\end{eqnarray*}
We complete the proof by using Lemma~\ref{lemma:d}, which states that the probability for either of the two conditions hold is no less than $1 - \delta$.
\section{Conclusions}
In this paper, we have studied the sample complexity of passive learning when the target expected  risk is given to the learner as  prior knowledge. The crucial fact about target risk assumption is that, it can be fully exploited by the learning algorithm and  stands in  contrast to most  common types of prior knowledges that usually enter into the generalization bounds and are often perceived as a rather  crude way to incorporate such assumptions. We showed that by explicitly employing the target risk $\ep$ in a properly designed stochastic optimization algorithm, it is possible to attain the given target risk $\ep$ with a logarithmic sample complexity $\log \left(\frac{1}{\ep}\right)$,  under the assumption that the loss function is both strongly convex and smooth.

There are various directions for future research. The current study is restricted to the parametric setting where the hypothesis space is of finite dimension. It would be interesting to see how to achieve a logarithmic sample complexity in a non-parametric setting where hypotheses lie in a functional space of infinite dimension. Evidently, it is impossible to extend the current algorithm for the non-parametric setting;  therefore additional analysis tools are needed to address the challenge of infinite dimension arising from the non-parametric setting.  It is also an interesting problem to relate target risk assumption we made here to the low noise margin condition which is often made in active learning for binary classification since both settings appear to share the same sample complexity. However it is currently unclear how to derive a connection between these two settings. We believe this issue is worthy of further exploration and leave it as an open problem. 
\label{sec:conclusion}

\newpage
\appendix
\section{Proof of Lemma~\ref{lemma:d}}
\label{app:lemma1}
The proof is based on the Bernstein inequality for martingales~(see, e.g., \cite{Cesa-Bianchi:2006:prediction}).
\begin{lemma} \label{thm:bernstein} (Bernstein inequality for martingales). Let $X_1, \ldots , X_n$ be a bounded martingale difference sequence with respect to the filtration $\F = (\F_i)_{1\leq i\leq n}$ and with $\|X_i\| \leq M$. Let $S_i = \sum_{j=1}^i X_j$  be the associated martingale. Denote the sum of the conditional variances by
\[
    \Sigma_n^2 = \sum_{t=1}^n \E\left[X_t^2|\F_{t-1}\right]
\]
Then for all constants $\kappa$, $\nu > 0$,
\[
\Pr\left[ \max\limits_{i=1, \ldots, n} S_i > \rho \mbox{ and } \Sigma_n^2 \leq \nu \right] \leq \exp\left(-\frac{\rho^2}{2(\nu + M\rho/3)} \right)
\]
and therefore,
\[
    \Pr\left[ \max\limits_{i=1,\ldots, n} S_i > \sqrt{2\nu \rho} + \frac{\sqrt{2}}{3}M\rho \mbox{ and } \Sigma_n^2 \leq \nu \right] \leq e^{-\rho}.
\]
\end{lemma}
\begin{proof}[of Lemma~\ref{lemma:d}] Define martingale difference $d_k^t = \left\langle \w_k^t - \w_*, \E_t[\v_k^t] - \v_k^t\right\rangle$  and martingale $D_k = \sum_{t=1}^{T_1} d_k^t$. Let $\Sigma_T^2$ denote the  conditional variance  as
\begin{eqnarray*}
    \Sigma_T^2 = \sum_{t=1}^{T_1} \E_{t}\left[(d_k^t)^2 \right]
    &\leq& \sum_{t=1}^{T_1} \E_t\left[\left\|\E_t[\v_k^t] - \v_k^t\right\|^2 \right]\|\w_k^t - \w_*\|^2 \\
    &\leq& \sum_{t=1}^T d \gamma_k^2  \|\w_k^t - \w\|^2 = d \gamma_k^2  W_k,
\end{eqnarray*}
which follows from the Cauchy's Inequality and the  definition of clipping. Define $M = \max\limits_{t} |d_k^t| \leq 2\sqrt{d}\gamma_k\Delta_k$.
To prove the inequality in Lemma~\ref{lemma:d}, we follow the idea of peeling process~\cite{koltchinskii-2011-oracle}. Since $W_k \leq 4R^2 T_1$, we have
\begin{eqnarray*}
\lefteqn{\Pr\left(D_k \geq 2\gamma_k\sqrt{W_k d \rho} + \sqrt{2}M\rho/3\right)} \\
& = & \Pr\left(D_k \geq 2\gamma_k\sqrt{W_k d \rho} + \sqrt{2}M\rho/3, W_k \leq 4R^2T_1\right) \\
& =  & \Pr\left(D_k \geq 2\gamma_k\sqrt{W_k d \rho} + \sqrt{2}M\rho/3, \Sigma_T^2 \leq \gamma_k^2 d W_k, W_k \leq 4R^2T_1 \right) \\
& \leq  & \Pr\left(D_k \geq 2\gamma_k\sqrt{W_k d \rho} + \sqrt{2}M\rho/3, \Sigma_T^2 \leq \gamma_k^2 d W_k, W_k \leq \ep T_1/(2\beta\mu) \right) \\
&  & + \sum_{i=1}^s \Pr\left(D_k \geq 2\gamma_k\sqrt{W_k d\rho} + \sqrt{2}M\rho/3, \Sigma_T^2 \leq \gamma_k^2 d W_k, \frac{\ep 2^{i-1} T_1}{2\beta\mu} < W_k  \leq \frac{\ep 2^i T_1}{2\beta\mu} \right) \\
& \leq  & \Pr\left(W_k \leq \frac{\ep  T_1}{2\beta\mu}\right) + \sum_{i=1}^s \Pr\left(D_k \geq \sqrt{\frac{\ep 2^{i+1}   T_1\gamma_k^2 d}{2\beta\mu}\rho} + \frac{\sqrt{2}}{3}M\rho, \Sigma_T^2 \leq \frac{ \ep 2^i T_1\gamma_k^2 d}{2\beta\mu}\right) \\
& \leq  & \Pr\left(W_k \leq \frac{\ep  T_1}{2\beta\mu}\right) + se^{-\rho},
\end{eqnarray*}
where $s$ is given by
\[
    s = \left\lceil \log_2 \frac{8\beta\mu R^2}{\ep} \right\rceil.
\]
The last step follows the Bernstein inequality for martingales. We complete the proof by setting $\rho= \ln(s/\delta)$ and using the fact that
\[
    2\gamma_k\sqrt{W_k \rho d} \leq \frac{1}{L}W_k + \gamma_k^2 \rho d L .
\]
\end{proof}
\section{Proof of Lemma~\ref{lemma:e}}
\label{app:lemma2}
To bound $E_k$, we need the following two lemmas.  The first lemma  bounds  the deviation of the expected value of a clipped random variable from  the original variable, in terms of its variance (Lemma A.2 from~\cite{DBLP:journals/corr/abs-1108-4559}).
\begin{lemma}
\label{lem:clip}
Let $X$ be a random variable, let $\Xt = \rm{clip}(X, C)$ and assume that $|\E[X]| \leq C/2$ for some $C > 0$. Then
\[
    |\E[\Xt] - \E[X]| \leq \frac{2}{C}\left|\rm{Var}[X]\right|
\]
\end{lemma}

Another key  observation used for bounding  $E_k$  is the  fact that for any non-negative $\beta$-smooth convex function,  we have the following self-bounding property. We note that this self-bounding property has been used in~\cite{srebro-2010-smoothness}  to get better (optimistic) rates of convergence for non-negative smooth losses.

\begin{lemma}
 \label{lem:smooth}
For any $\beta$-smooth non-negative function $f: \R\rightarrow \R$, we have $|f'(w)| \leq \sqrt{4\beta f(w)}$
\end{lemma}
As a simple proof,  first from the smoothness assumption,  by  setting $w_1 = w_2 - \frac{1}{\beta}f'(w_2)$ in (\ref{eqn:smoth})  and rearranging the terms we obtain $f(w_2) - f(w_1) \geq \frac{1}{2 \beta} | f'(w_2)|^2$.  On the other hand, from the convexity of loss function  we have $f(w_1) \geq f'(w_2) + \dd{f'(w_1)}{w_1 - w_2}$. Combining these inequalities and considering the fact that the function is non-negative gives the desired inequality. \\

\begin{proof}[of Lemma~\ref{lemma:e}] To apply the above lemmas, we write $e_k^t$ as
\begin{eqnarray*}
e_k^t & = & \sum_{i=1}^d \E_t\left[\ell'( \langle\w_k^t, \x_k^t \rangle, y_t)[\x_k^t]_i - \clip\left(\gamma_k, \ell'(\langle\w_k^t, \x_k^t \rangle, y_t)[\x_k^t]_i \right) \right] [\w_k^t - \w_*]_i
\end{eqnarray*}
In order to apply Lemma~\ref{lem:clip}, we check if the following condition holds
\begin{eqnarray}
    \gamma_k \geq 2\left|\E_t\left[\ell'\left(\langle \w_k^t, \x_k^t \rangle, y_t\right)[\x_k^t]_i\right]\right| \label{eqn:cond-1}
\end{eqnarray}
Since
\begin{eqnarray*}
& & \left|\E_t\left[\ell'\left(\langle \w_k^t, \x_k^t \rangle, y_t\right)[\x_k^t]_i\right]\right| \\
& \leq & \left|\E_t\left[\left\{\ell'\left(\langle \w_k^t, \x_k^t \rangle, y_t\right) - \ell'\left(\langle \w_*, \x_k^t \rangle, y_t\right)\right\}[\x_k^t]_i\right]\right| + \left|\E_t\left[\ell'\left(\langle \w_*, \x_k^t \rangle, y_t \right)[\x_k^t]_i\right]\right| \\
& \leq & \beta \|\w_k^t - \w_*\| \leq \beta \Delta_k
\end{eqnarray*}
where the last inequality follows from $\E_t\left[\ell'\left(\langle \w_*, \x_k^t \rangle, y_t\right)[\x_k^t]_i\right] = 0$ since $\w_*$ is the minimizer of $\lb(\w)$, we thus have
\[
    \gamma_k = 2\xi\beta\Delta_k \geq 2\beta\Delta_k \geq 2\left|\E_t\left[\ell'\left(\langle \w_k^t, \x_k^t \rangle, y_t\right)[\x_k^t]_i\right]\right|
\]
where $\xi \geq 1$, implying that the condition in (\ref{eqn:cond-1}) holds. Thus, using Lemma~\ref{lem:clip}, we have
\begin{eqnarray*}
e_k^t &\leq& \sum_{i=1}^d \left|[\w_k^t - \w_*]_i\right|\frac{1}{\gamma_k}\E_t\left[\left(\ell'(\langle\w_k^t, \x_k^t \rangle, y_t)[\x_k^t]_i\right)^2\right] \\
 &\leq& \frac{2\|\w_k^t - \w_*\|_{\infty}}{\gamma_k} \E_t\left[\left(\ell'(\langle\w_k^t, \x_k^t \rangle, y_t)\right)^2\right]
\end{eqnarray*}
Using Lemma~\ref{lem:smooth} to upper bound the right hand side, we further simplify the above bound for $e_k^t$ as
\begin{eqnarray*}
e_k^t  &\leq& \frac{8\beta\|\w_k^t - \w_*\|_{\infty}}{\gamma_k}\E_t\left[\ell\left(\langle \w_k^t, \x_k^t \rangle, y_t \right) \right] \\
&=& \frac{8\beta\|\w_k^t - \w_*\|_{\infty}}{\gamma_k}\lb(\w_k^t)\\
&\leq&  \frac{8\beta \Delta_k}{\gamma_k}\lb(\w_k^t)\\
 &= & \frac{4}{\xi}\lb(\w_k^t)
\end{eqnarray*}
where the second inequality follows from $\|\w_k^t - \w_*\|_{\infty} \leq \|\w_k^t - \w_*\| \leq \Delta_k$. 
Therefore we obtain
\begin{eqnarray*}
E_k  =  \sum_{t=1}^{T_1} e_k^t \leq  \frac{4}{\xi}\sum_{t=1}^{T_1} \lb(\w_k^t) &=& \frac{4}{\xi}\sum_{t=1}^{T_1} \lb(\w_*) + \frac{4}{\xi}\sum_{t=1}^{T_1} \lb(\w_k^t) - \lb(\w_*) \\
& \leq & \frac{4T_1}{\xi} \lb(\w_*) + \frac{4\beta}{\xi}\sum_{t=1}^{T_1} \|\w_k^t - \w_*\|^2\\ &=& \frac{4T_1}{\xi} \lb(\w_*) + \frac{4\beta}{\xi}W_k,
\end{eqnarray*}
where the second inequality follows from the smoothness assumption of $\lb(\w)$.
\end{proof}

%


\bibliographystyle{plain}
\bibliography{passive_target_risk}
\end{document}